\documentclass[twoside,11pt]{article}

\usepackage{jmlr2e}

\usepackage{amsmath,color}
\usepackage[utf8]{inputenc}
\usepackage[dvipsnames]{xcolor}
\usepackage{soul}
\usepackage{enumitem}
\usepackage[parfill]{parskip}
\usepackage{hyperref}
\usepackage[hang,flushmargin]{footmisc}
\usepackage{scrextend}

\usepackage[nottoc,notlot,notlof]{tocbibind}

\hypersetup{
	colorlinks,
	citecolor=blue,
	filecolor=blue,
	linkcolor=blue,
	urlcolor=blue
}
\newtheorem{claim}[theorem]{Claim}
\newtheorem{lem}[theorem]{Lemma}

\newtheorem{prop}[theorem]{Proposition}

\newtheorem{cor}[theorem]{Corollary}

\newcommand{\err}{\mathrm{err}}
\newcommand{\emperr}{\widehat{\err}}
\newcommand{\pr}[2]{\underset{#1}{\mathbb{P}}\left[ #2 \right]}
\newcommand{\ex}[2]{\underset{#1}{\mathbb{E}}\left[ #2 \right]}
\newcommand{\eps}{\varepsilon}
\newcommand{\tvd}{\mathrm{d_{TV}}}
\newcommand{\cX}{\mathcal{X}}
\newcommand{\cY}{\mathcal{Y}}
\newcommand{\cU}{\mathcal{U}}
\newcommand{\cZ}{\mathcal{Z}}
\newcommand{\cD}{\mathcal{D}}
\newcommand{\cB}{\mathcal{B}}
\newcommand{\cO}{\mathcal{O}}
\newcommand{\rA}{\rightarrow}

\newcommand{\supp}{\mathsf{Supp}}

\jmlrheading{?}{2018}{?-?}{10/27}{?/?}{?}{Raef Bassily, Shay Moran, Ido Nachum, Jonathan Shafer and Amir Yehudayoff}

\ShortHeadings{Learners that Use Little Information}{Bassily, Moran, Nachum, Shafer and Yehudayoff}
\firstpageno{1}

\begin{document}
	
\global\long\def\kl{\mathsf{KL}}
\global\long\def\aa{\mathcal{\alpha}}
\global\long\def\cA{\mathcal{A}}
\global\long\def\bb{\mathcal{B}}
\global\long\def\cc{\mathcal{C}}
\global\long\def\DD{\mathcal{D}}
\global\long\def\ff{\mathcal{F}}

\global\long\def\hh{\mathcal{H}}
\global\long\def\ii{\mathcal{I}}
\global\long\def\jj{\mathcal{J}}
\global\long\def\kk{\mathcal{K}}
\global\long\def\lll{\mathcal{L}}
\global\long\def\mm{\mathcal{M}}
\global\long\def\nn{\mathcal{N}}
\global\long\def\oo{\mathcal{O}}
\global\long\def\pp{\mathcal{P}}
\global\long\def\qq{\mathcal{Q}}
\global\long\def\rr{\mathcal{R}}
\global\long\def\ss{\;}
\global\long\def\uu{\mathcal{U}}
\global\long\def\vv{\mathcal{V}}
\global\long\def\ww{\mathcal{W}}
\global\long\def\xx{\mathcal{X}}
\global\long\def\yy{\mathcal{Y}}
\global\long\def\zz{\mathcal{Z}}
\global\long\def\k{\mathcal{\kappa}}
\global\long\def\r{\mathcal{\rho}}
\global\long\def\m{\mathcal{\mu}}
\global\long\def\n{\nu}

\global\long\def\Wa{W_{\mathcal{A}}}
\global\long\def\CC{\mathbb{C}}
\global\long\def\Dd{\mathbb{D}}
\global\long\def\EE{\mathbb{\mathbb{E}}}
\global\long\def\FF{\mathbb{F}}
\global\long\def\KK{\mathbb{K}}
\global\long\def\LL{\mathbb{L}}
\global\long\def\NN{\mathbb{N}}
\global\long\def\PP{\mathbb{P}}
\global\long\def\QQ{\mathbb{Q}}
\global\long\def\RR{\mathbb{R}}
\global\long\def\SS{\mathbb{S}}
\global\long\def\TT{\mathbb{T}}
\global\long\def\UU{\mathbb{U}}
\global\long\def\VV{\mathbb{V}}
\global\long\def\WW{\mathbb{W}}
\global\long\def\XX{\mathbb{X}}
\global\long\def\YY{\mathbb{Y}}
\global\long\def\ZZ{\mathbb{Z}}
\global\long\def\ne{\neq}
\global\long\def\ge{\leq}
\global\long\def\e{\varepsilon}
\global\long\def\dd{\delta}
\global\long\def\ne{\neq}
\global\long\def\sign{\stackrel[i=1]{n}{\sum}}
\global\long\def\sigm{\stackrel[i=1]{m}{\sum}}
\global\long\def\sigmn{\stackrel[i=m+1]{n}{\sum}}
\global\long\def\sigN{\stackrel[i=1]{N}{\sum}}
\global\long\def\se{\geq}

\global\long\def\con{\subset}
\global\long\def\sec{\cap}
\global\long\def\un{\cup}
\global\long\def\inf{\infty}
\global\long\def\do{_{1}}
\global\long\def\dt{_{2}}
\global\long\def\a{\rightarrow}
\global\long\def\x{\times}
\global\long\def\s{\curvearrowright}
\global\long\def\at{\mapsto}
\global\long\def\br{\left(\right)}
\global\long\def\c{\circ}
\global\long\def\d{\cdot}

\newcommand{\mynote}[1]{{#1}}

\definecolor{DarkPurple}{rgb}{0.7,0.2,0.4}
\newcommand{\rnote}[1]{\mynote{\color{DarkPurple} Raef: {#1}}}
\newcommand{\new}[1]{{\color{red} #1}}
\newcommand{\shay}[1]{\mynote{\color{red} Shay: {#1}}}

\title{Learners that Use Little Information}

\author{\name Raef Bassily \email bassily.1@osu.edu \\
       \addr Department of Computer Science and Engineering\\
       The Ohio State University\\
       Columbus, OH
       \AND
       \name Shay Moran \email shaymoran1@gmail.com \\
       \addr School of Mathematics\\
       Institute for Advanced Study\\
       Princeton, NJ
       \AND
       \name Ido Nachum \email idon@tx.technion.ac.il \\
       \addr Department of Mathematics\\
       Technion-IIT\\
       Haifa, Israel
       \AND
       \name Jonathan Shafer \email shaferjo@berkeley.edu \\
       \addr Computer Science Division \\
       University of California, Berkeley\\
       Berkeley, CA
       \AND
       \name Amir Yehudayoff \email amir.yehudayoff@gmail.com \\
       \addr Department of Mathematics\\
       Technion-IIT\\
       Haifa, Israel}

\editor{?}

\maketitle

\begin{abstract}
	We study learning algorithms that are restricted to using a small amount of information from their input sample. We introduce a category of learning algorithms we term {\em $d$-bit information learners}, which are algorithms whose output conveys at most $d$ bits of information of their input. A central theme in this work is that such algorithms generalize. 
	
	We focus on the  learning capacity of these algorithms, and prove sample complexity bounds with tight dependencies
	on the confidence and error parameters. We also observe connections with well studied notions 
	such as sample compression schemes, Occam's razor, PAC-Bayes and differential privacy.
	
	We discuss an approach that allows us to prove upper bounds on the amount of information that algorithms reveal about their inputs, and also provide a lower bound by showing a simple concept class for which every (possibly randomized) empirical risk minimizer must reveal a lot of information. On the other hand, we show that in the distribution-dependent setting every VC class has empirical risk minimizers that do not reveal a lot of information.
\end{abstract}


\section{Introduction}

The amount of information that an algorithm uses is a natural and important quantity to study. A central idea that this paper revolves around is that a learning algorithm that only uses a small amount of information from its input sample will generalize well. The amount of information used can be quantified by
\[
I(\cA(S);S)
\]
which is the mutual information between the output of the algorithm $\cA(S)$ and the input sample $S$. With this quantity we define a new class of learning algorithms termed {\em $d$-bit information learners}, which are learning algorithms in which the mutual information is at most $d$. This definition naturally combines notions from information theory and from learning theory~\citep{comp->learn, learn->comp}. It also relates to privacy, because one can think of this definition as a bound on the information that an algorithm leaks or reveals about its potentially sensitive training data.

\subsection{Results}
The main contributions of the current work are as follows (for definitions see Section~\ref{sec:prel}. For related work, see Section~\ref{related-work}).

\paragraph{Low information yields generalization.}
Our work stems from the intuition that a learning algorithm that only uses a small amount of information from its input will generalize well. We formalize this intuition in Theorem \ref{thm:info-comp}, that roughly states that
\[
\pr{}{| \text{true error $-$ empirical error} | > \e}=O\left(\frac{I(\cA(S);S)}{m\e^{2}}\right) ,
\]
where $m$ is the number of examples in the input sample.
We provide four different proofs of this statement, each of which emphasizes a different perspective 
of this phenomenon (see Section~\ref{app:proofs-gen}).

\paragraph{Sharpness of the sample complexity bound.}
Theorem \ref{thm:info-comp} entails that to achieve an error of $\e$ with confidence $\delta$, it is sufficient to use   
\[
m = \Omega\left(\frac{I(\cA(S);S)}{\eps^2 \cdot \delta}\right)
\]

examples. This differs from results in well-known settings such as learning hypothesis classes of finite VC dimension, where $m$ only grows logarithmically with $1/\delta$. Nonetheless, we prove that this bound is sharp (Section \ref{sec:sharp}). In particular, we show the existence of a learning problem and an $O(1)$-bit information learner that has a true error of at least $\frac{1}{2}$ with probability of at least $\frac{1}{m}$, where $m$ is the size of the input sample.

\paragraph{A lower bound for mutual information.}
In Section \ref{sec:lower} we show that for the simple class of thresholds, every (possibly randomized) proper ERM must reveal at least 
\[
\Omega \left( \frac{\log \log N}{m^2} \right)
\]
bits of information, where $N$ is the size of the domain. This means that even in very simple settings, learning may not always be possible if we restrict the information used by the algorithm. However, this does not imply the non-existence of bounded information learners that are either non-consistent or non-proper, an issue we leave open for future work.

\paragraph{Upper bounds for mutual information.}
Section \ref{sec:upper-method} provides a method for upper bounding the amount of information that algorithms reveal. We also define a generic learner $\cA_\hh$ for a concept class $\hh$, and show that in a number of natural cases this algorithm conveys as little information as possible (up to some constant). This generic learner is proper and consistent (i.e.\ an ERM); it simply outputs a uniformly random hypothesis from the set of hypotheses that are consistent with the input sample. However, we show that in other simple cases, this algorithm has significantly higher mutual information than necessary.

\paragraph{The distribution-dependent setting.}
We also consider an alternative setting, in Section \ref{sec:distribution-dependent}, in which the distribution over the domain is known to the learner. Here, for any concept class with finite VC-dimension $d$ and for any distribution on the data domain there exists a learning algorithm that outputs with high probability an approximately correct function from the concept class, such that the mutual information between the input sample and the output is $O\left(d\log(m)\right)$. In contrast with the abovementioned lower bound, the information here does not grow with the size of the domain.

\paragraph{Contrast with pure differential privacy.}
Corollary \ref{cor:privacy-separation} provides a separation between differential privacy and bounded mutual information. For the class of point functions ${\cal PF}$, it is known that any pure differentially private algorithm that properly learns this class must require a number of examples that grows with the domain size \citep{BeiNS10}. On the other hand, we show that the generic ERM learner $\cA_{\cal PF}$ leaks at most $2$ bits of information and properly learns this class with optimal PAC-learning sample complexity.

\subsection{Related Work}\label{related-work}

\paragraph{Sample compression schemes.} $d$-bit information learners resemble the notion of sample compression schemes~\citep{comp->learn}. Sample compression schemes correspond to learning algorithms whose output hypothesis is determined by a small subsample of the input. For example, {\emph support vector machines} output a separating hyperplane that is determined by a small number of support vectors.

Both sample compression schemes and information learners quantify (in different ways) the property of limited dependence between the output hypothesis and the input sample. It is therefore natural to ask how these two notions relate to each other.

It turns out that not every sample compression scheme of constant size also leaks a constant number of bits. Indeed, in Section~\ref{sec:lower} it is shown that there is no empirical risk minimizer (ERM) for thresholds that is an $O(1)$-bits information learner.\footnote{Here and below $O(\cdot)$, $\Omega(\cdot)$ and $\Theta(\cdot)$ mean up to some multiplicative universal constants.} On the other hand, there is an ERM for this class that is based on a sample compression scheme of size $O(1)$.

\paragraph{Occam's razor.} Theorem \ref{thm:info-comp} extends the classical Occam's razor generalization bound~\citep{occam}, which states the following: Assume a fixed encoding of hypotheses in $\hh$ by bit strings. The {\em complexity} of a hypothesis  is the bit-length of its encoding. A learning algorithm for $\hh$ is called an \emph{Occam-algorithm with parameters $c,\alpha$} if for every realizable sample of size $m$ it produces a consistent hypothesis of complexity at most $n^c m^{\alpha}$, where $n$ is the complexity of some hypothesis in $\hh$ that is consistent with the sample. 

\begin{addmargin}[1em]{2em}
	\begin{theorem}[\citealt*{occam}] 
		Let $\cA$ be an Occam-algorithm with parameters $c\se1$ and $0\ge\aa<1$.
		Let $\DD$ be  a realizable distribution,
		let $f\in\hh$ be such that $\err(f ;\DD)=0$, and let $n$ denote the complexity of $f$.
		Then, 
		\[\Pr_{S\sim \DD^m}\bigl[\err\bigl(\cA(S) ; \DD\bigl) \geq \eps\bigr]\leq \delta,\]
		as long as $m$ is at least 
		$\Omega\left(\frac{\log(\frac{1}{\dd})}{\e}+\left(\frac{n^{c}}{\e}\right)^{1/\left(1-\aa\right)}\right) .$
	\end{theorem}
\end{addmargin}

To relate Occam's razor to Theorem \ref{thm:info-comp}, observe that an Occam-algorithm is in particular a $O(n^c m^{\alpha})$-bit information learner (since its output hypothesis is encoded by $O(n^c m^{\alpha})$ bits), which implies that the probability of it outputting a function with true error more than $\e$ is at most $O \left( \frac{n^c m^{\alpha}}{m\e^{2}} \right)$.
The bound can be improved by standard confidence-boosting techniques (see Appendix~\hyperref[apndx:boost]{I}).

\paragraph{Mutual information for controlling bias in statistical analysis.} The connection between mutual information and statistical bias has been recently studied in \cite{Russo-Zhou16} in the context of \emph{adaptive} data analysis. In adaptive statistical analysis, the analyst conducts a sequence of analysis steps, where the choice and structure of each step depends adaptively on the outcomes of the previous ones. Some of the results of \cite{Russo-Zhou16} have been recently improved by \cite{RX17}.

\paragraph{Differential privacy and generalization.} Differential privacy, introduced by \cite{DMNS06}, is a rigorous notion of privacy enabling a strong guarantee that data holders may provide to their sources. \emph{Pure}\footnote{As opposed to a more relaxed notion known as \emph{approximate} differential privacy (see Section~\ref{sec:dp} for a precise definition).} differential privacy implies a bound on mutual information~\citep{MMP+10}.

The role of differential privacy in controlling overfitting has been recently studied in
several works \cite[e.g.][]{DFH+15, BNS+16, RRST16, BST14}. The authors of \cite{BNS+16} provide a treatment of differential privacy as a notion of distributional stability, and a tight characterization of the generalization guarantees of differential privacy.

\paragraph{Max-information and approximate max-information:} \cite{DFH+b15} introduced and studied the notions of max-information -- a stronger notion than mutual information -- and its relaxation, approximate max-information.\footnote{Unlike max-information, the relaxed notion of approximate max-information is not directly related to mutual information; that is, boundedness of one does not necessarily imply the same for the other.} They showed that these notions imply generalization and that pure differentially private
algorithms exhibit low (approximate) max-information. 
\cite{RRST16} showed that approximate differentially private algorithms also have low approximate max-information, and that the notion of approximate max-information captures the generalization properties (albeit with slightly worse parameters) of differentially private algorithms (pure or approximate).

\paragraph{Connections to approximate differential privacy:} 

\cite{De11} has shown that the relaxed notion of \emph{approximate} differential privacy does not necessarily imply bounded mutual information. In \cite{MMP+10}, it was also shown that if the dataset entries are independent, then approximate differential privacy implies a (weak) bound on the mutual information. Such a bound has an explicit dependence on the domain size, which restricts its applicability in general settings. Unlike the case of pure differential privacy, an exact characterization of the relationship between mutual information and approximate differential privacy algorithms is not fully known even when the dataset distribution is i.i.d.

\cite{BNSV15} showed that the sample complexity of properly learning thresholds (in one dimension) under approximate differential privacy is $\Omega(\log^*(N))$, {where $N$ is the domain size}. Hence, their result asserts the impossibility of this task for infinite domains. In this work, we show a result of a similar flavor (albeit of a weaker implication) for the class of bounded information learners. Specifically, for the problem of proper PAC-learning of thresholds over a domain of size $N$, we show that the mutual information of any proper learning algorithm (deterministic or randomized) that outputs a threshold that is \emph{consistent} with the input sample is $\Omega(\log \log N)$. This result implies that there are no consistent proper bounded information learners for thresholds over infinite domains.

\section{Preliminaries}\label{sec:prel}

\subsection{Learning}

We start with some basic terminology from statistical learning (for a textbook see \citealt*{SS-text}). Let $\xx$ be a set called \emph{the domain}, $\yy=\{0,1\}$ be the \emph{label-set}, and $\zz=\xx\times\yy$ be the \emph{examples domain}. A \emph{sample} $S = ((x_1,y_1),\ldots,(x_m,y_m))\in\zz^m$ is a sequence of examples. A function $h:\xx\to\yy$ is called a \emph{hypothesis} or a \emph{concept}.

Let $\DD$ be a distribution over $\zz$. The \emph{error} of a hypothesis $h$ with respect to $\DD$ is defined by $\err(h; \DD)=  \EE_{(x,y)\sim \DD} 1[h(x)\neq y]$. Let $S = ((x_1,y_1),\ldots,(x_m,y_m))$ be a sample. The \emph{empirical error} of $h$ with respect to $S$ is defined by $\emperr(h; S)= \frac{1}{m}\sum_{i=1}^{m}  1[h(x)\neq y]$.

A \emph{hypothesis class} $\hh$ is a set of hypotheses. A distribution $\DD$ is \emph{realizable} by $\hh$ if there is $h\in \hh$ with $\err(h; \DD)=0$. A sample $S$ is \emph{realizable} by $\hh$ if there is $h\in H$ with $\emperr(h; S)=0$.

A {\em learning algorithm}, or a {\em learner} is a (possibly randomized) algorithm $\cA$ that takes a sample $S$ as input and outputs a hypothesis, denoted by $\cA(S)$. We say that $\cA$ \emph{learns}\footnote{In this paper we focus on learning in the realizable case.} $\hh$ if for every $\eps,\delta>0$ there is a finite bound $m=m(\eps,\delta)$ such that for every $\hh$-realizable distribution $\DD$, 
\[
\Pr_{S\sim \DD^m}[\err\bigl(\cA(S) ;\DD\bigr) \geq \eps] \leq \delta.
\]
$\eps$ is called the \emph{error} parameter, and $\delta$ the \emph{confidence} parameter. $\cA$ is called {\em proper} if $\cA(S)\in \hh$ for every realizable $S$, and it is called {\em consistent} if $\emperr\bigl(\cA(S); S\bigr)=0$ for every realizable $S$.

\subsection{Information Theoretic Measures}

Information theory studies the quantification and communication of information. In this work, we use the language of learning theory combined with information theory to define and study a new type of learning theoretic compression.
Here are standard notions from information theory (for more background see the textbook \citealt*{book}).

Let $Y$ and $Z$ be two discrete random variables.
The entropy of $Y$ measures the number of bits required to encode
$Y$ on average.
\medskip

\begin{definition}[Entropy]
	The entropy of $Y$ is defined as
	$$H\left(Y\right)= - \sum_{y}\Pr\left(Y=y\right)\log\left(\Pr\left(Y=y\right)\right),$$
	where $\log = \log_2$ and
	by convention $0\log0 =0$.
\end{definition}

The mutual information between $Y$ and $Z$ is (roughly speaking) a measure
for the amount of random bits $Y$ and $Z$ share on average. 
It is also a measure of their independence;
for example $I\left(Y;Z\right)=0$ iff $Y$ and $Z$ are independent. 

\medskip

\begin{definition}[Mutual information]
	The mutual information between $Y$ and $Z$ is
	defined to be
	$$I\left(Y;Z\right)=H(Y)+H(Z)-H(Y,Z).$$ 
\end{definition}

The Kullback-Leibler divergence
between two measures $\m$ and $\n$ is a useful measure for the ``distance'' between them (it is not a metric and may be infinite).

\medskip
\begin{definition}[KL-divergence]
	The KL-divergence between two measures $\mu$ and $\nu$
	on $\cX$ is
	$$\kl\left(\mu||\n\right)=\sum_{x}\mu(x)\log \frac{\mu(x)}{\nu(x)}$$
	where $0 \log \frac{0}{0}=0$.
\end{definition}

Mutual information can be written as the following
KL-divergence:
\[I\left(Y;Z\right)=\kl\left(p_{Y,Z}||p_Y \cdot p_Z \right),\]
{where $p_{Y,Z}$ is the joint distribution of the pair $(Y,Z)$, 
	and $p_Y \cdot p_Z$ is the product 
	of the marginals $p_Y$ and $p_Z$.}

\subsection{Differential Privacy}\label{sec:dp}

Differential privacy \citep{DMNS06} is a standard notion for statistical data privacy. Despite the connotation perceived by the name, differential privacy is a distributional stability condition that is imposed on an algorithm performing analysis on a dataset. Algorithms satisfying this condition are known as \emph{differentially private} algorithms. There is a vast literature on the properties of this class of algorithms and their design and structure \citep[see e.g.,][for an in-depth treatment]{dp-text}. 

\medskip

\begin{definition}[Differential privacy]
	Let $\cX, \cZ$ be two sets, and let $m\in\NN$. Let $\alpha >0, \beta\in [0,1)$. An 
	algorithm $\cA:\cX^m\rA \cZ$ is said to be $(\alpha, \beta)$-differentially private if for all datasets $S, S' \in\cX^m$ that differ in exactly one entry, 
	and all measurable subsets $\cO\subseteq \cZ$, we have 
	\begin{align}
	\pr{\cA}{\cA(S)\in\cO}&\leq e^{\alpha}~\pr{\cA}{\cA(S')\in\cO} + \beta\nonumber
	\end{align}
	where the probability is taken over the random coins of $\cA$. 
	
	When $\beta=0$, the condition is sometimes referred to as \emph{pure} differential privacy (as opposed to \emph{approximate} differential privacy when $\beta>0$.) 
\end{definition}

The general form of differential privacy entails two parameters: $\alpha$ which is typically a small constant 
and $\beta$ which in most applications is of the form $\beta = o(1/m)$.

Differential privacy has been shown to provide non-trivial generalization guarantees especially in the adaptive settings of statistical analyses \cite[see e.g.,][]{DFH+15, BNS+16, DFH+b15}. In the context of (agnostic) PAC-learning,  there has been a long line of work \citep[e.g.][]{KLNRS08, BeiNS10, BeiNS13, FX14, BNSV15} that studied differentially private learning and the characterization of the sample complexity of private learning in several settings. However, the picture of differentially private learning is very far from complete and there are still so many open questions. \cite{Vad_survey_DP16} gives a good survey on the subject.

\section{$d$-Bit Information Learners}

Here we define learners that use little information from their input.\footnote{In this text we focus on Shannon's mutual information, but other notions of divergence may be interesting to investigate as well.} We start by setting some notation. Let $\cA$ be a {(possibly randomized)} learning algorithm. For every sample $S \in \left(\cX\times \{0,1\} \right)^{m}$, let $P_{h | S}(\cdot)$ denote the conditional distribution function of the output of the algorithm $\cA$ given that its input is $S$. When $\cA$ is deterministic, $P_{h | S}$ is a degenerate distribution. For a fixed distribution $\DD$ over examples and $m\in\NN$, let $P_h(\cdot)$ denote the marginal distribution of the output of $\cA$ when it takes an input sample of size $m$ drawn i.i.d.~from $\DD$, i.e.\ $P_h(f)=\ex{S\sim \DD^m}{P_{h|S}(f)}$ for every function $f$.

\medskip

\begin{definition}[Mutual information of an algorithm]\label{defn:mut-inf-of-alg}
	We say that $\cA$ has mutual information of at most $d$ bits (for sample size $m$) with respect to a distribution $\DD$ if
	\[
	I(S; \cA(S))\leq d
	\]
	where $S\sim\DD^m$.
\end{definition}

\medskip

\begin{definition}[$d$-bit information learner]\label{defn:inf-comp}
	A learning algorithm $\cA$ for $\hh$ is called a $d$-bit information learner if it has mutual information
	of at most $d$ bits with respect to every realizable distribution ($d$ can depend on the sample size).
\end{definition}

\subsection{Bounded Information Implies Generalization}\label{sec:gen}

The following theorem quantifies the generalization guarantees of $d$-bit information learners.

\medskip

\begin{theorem}
	\label{thm:info-comp}
	Let $\cA$ be a learner that has mutual information of at most $d$ bits with a distribution $\DD$,
	and let $S\sim \DD^m$. Then, for every $\eps > 0$,
	\[
	\pr{\cA, S}{\lvert \emperr\left(\cA(S); S\right)-\err(\cA(S); \DD)\rvert > \e}<\frac{d+1}{2m\e^{2}-1}
	\]
	where the probability is taken over the randomness in the sample $S$ and the randomness of $\cA$.
\end{theorem}

Theorem~\ref{thm:info-comp} states a simple and basic property, and is proved in section~\ref{app:proofs-gen}
(below we provide a proof sketch for deterministic algorithms). 

In particular, if a class $\hh$ admits a $o(m)=d$-bit information learner then the class is PAC learnable. Also, some of the proofs will go through for multi-class classification with every bounded loss function.

The fact that the sample complexity bound that follows from the theorem is sharp is proved in Section \ref{sec:sharp}. We mention that the dependence on $\e$ can be improved in the realizable case; if the algorithm always outputs a hypothesis with empirical error $0$ then the bound on the right hand side can be replaced by $O\left( \frac{d+1}{m\e -1} \right)$.
As in similar cases, the reason for this difference stems from the fact that estimating the bias of a coin up to an additive
error $\eps$ requires $m \approx \frac{1}{\eps^2}$ samples, but if the coin falls on heads with probability $\eps$ then the chance of seeing $m$ tails in a row is $(1-\eps)^m \approx e^{-m \eps}$.

\subsubsection*{Proof Sketch for Deterministic Algorithms}\label{sec:proof-sketch}

Here we sketch a proof of Theorem~\ref{thm:info-comp} for deterministic algorithms. When $\cA$ is deterministic, we have
\[
I = I(S;\cA(S)) = H(\cA(S)).
\]
Let $P_{h}$ denote the distribution of $\cA(S)$. Let $\hh_0$ be the set of hypotheses $f$ so that $P_{h}(f) \geq 2^{I/\delta}$. By Markov's inequality, $P_{h}(\hh_0) \geq 1-\delta$. In addition, the size of $\hh_0$ is at most $2^{I/\delta}$. So Chernoff's inequality and the union bound imply that for every $f \in \hh_0$ the empirical error is close to the true error for $m \approx \frac{I}{\eps^2 \delta}$ (with probability at least $1-\delta$).

\section{Proofs that Bounded Information Implies Generalization}\label{app:proofs-gen}

In this paper, we prove the statement in Theorem~\ref{thm:info-comp} via different approaches (some of the arguments are only sketched). We provide four different proofs of this statement, each of which highlights a different general idea. 

The first proof is based on an information theoretic lemma, which roughly states that if the KL-divergence between two measures $\mu$ and $\nu$ is small then $\mu(E)$ is not much larger than $\nu(E)$ for every event $E$. The nature of this proof strongly resembles the proof of the PAC-Bayes bounds \citep{SS-text}, and indeed a close variant of the theorem can be derived from these standard bounds as well (see proof IV). The second proof is based on a method to efficiently ``de-correlate'' two random variables in terms of their mutual information; roughly speaking, this implies that an algorithm of low mutual information can only generate a small number of hypotheses and hence does not overfit. The third proof highlights an important connection between low mutual information and the stability of a learning algorithm.  The last proof uses the PAC-Bayes framework. Following is the first proof, see appendices \ref{proof-ii}, \ref{proof-iii} and \ref{proof-iv} for the other proofs.

\subsection{Proof I: Mutual Information and Independence}

The first proof of that we present uses the following lemma, which allows to control a distribution $\m$ by a distribution $\n$ as long as it is close to it in KL-divergence. The proof technique is similar to a classical technique in Shannon's information theory, e.g., in \cite{arutyunyan'68}.

\medskip

\begin{lem}
	\label{lem:lemmaDiv}
	Let $\m$ and $\n$ be probability distributions
	on a finite set $\cX$ and let $E \subseteq \cX$.
	Then, 
	$$\m\left(E\right)\ge \frac{\kl\left(\m||\n\right)+1}{\log\left(1/\n\left(E\right)\right)}.$$
\end{lem}
The lemma enables us to compare between events of small probability: if $\n\left(E\right)$ is small then $\m\left(E\right)$ is also small, as long as $\kl(\mu||\nu)$ is not very large.

The bound given by the lemma above is tight, as the following example shows. Let $\cX=[2n]$ and let $E=[n]$. For each $x \in \cX$, let
\[
\m\left(x\right)=\frac{1}{2n}
\]
and let
\[
\n\left(x\right)=\begin{cases}
1/n^{2} & x \in E , \\
\left(n-1\right)/n^{2} & x \not \in E . \end{cases}
\]
Thus, 
$\m\left(E\right)=\frac{1}{2}$ and $\n\left(E\right)=\frac{1}{n}$.
But on the other hand 
\begin{align*}
\kl\left(\m||\n\right)
& =\frac{1}{2}\log\left(\frac{n}{2}\right)+\frac{1}{2}\log\left(\frac{1}{2}\right)+o\left(1\right) ,
\end{align*}
so
\[
\frac{1}{2}=\m\left(E\right)\se\lim_{n\a\inf}\frac{\kl\left(\m||\n\right)+1}{\log\left(1/\n\left(E\right)\right)}=\frac{1}{2} .
\]
Similar examples can be given when $\kl\left(\m||\n\right)$ is constant.

We now change the setting to allow it to apply more naturally to $d$-bit information learners.

\medskip

\begin{lem}
	\label{lem:lemmaDiv2}
	Let $\mu$ be a distribution on the space $\cX\times \cY$ and let $E$ be an event that satisfies 
	\[
	\mu_X (E_y)<\alpha
	\]
	for all $y\in \cY$, where $\mu_X$ is the marginal distribution of $X$ and $E_y
	= \{x : (x,y) \in E\}$ is a fiber of E over y. Then
	\[
	\mu (E) \leq  \frac{I(X;Y)+1}{\log (1/\alpha)}
	\]
\end{lem}

The lemma enables us to bound the probability  of an event $E$, if we have a bound on the probability of its fibers over $y$ measured with the marginal distribution of $X$.

This lemma can be thought of as a generalization of the extremal case where $X$ and $Y$ are independent (i.e.\ $I(X,Y)=0$). 
In this case,  the lemma corresponds to the following geometric statement in the plane: if the width of every $x$-parallel fiber of a shape is at most $\alpha$ and its height is bounded by 1 then its area is also at most $\alpha$. The bound given by the above lemma is $\frac{1}{\log(1/\alpha)}$, which is weaker, but the lemma applies more generally when $I(X,Y)>0$. In fact, the bound is tight when the two variables are highly dependent; e.g. $X=Y$ and $X \sim U([n])$. In this case, the probability of the diagonal $E$ is $1$, while $I(X;Y)=\log (n)$ and $\alpha=1/n$. So indeed $1=\mu (E)\approx \frac{\log (n)+1}{\log (n)}$.

We now use this lemma to prove the theorem.

\begin{proof}[Proof of Theorem~\ref{thm:info-comp}]
	Let $\mu$ be the distribution on pairs $(S,h)$ where $S$ is chosen i.i.d.~from $\DD$ and $h$ is the output of the algorithm given $S$. Let $E$ be the event of error; that is,
	\[
	E=\left\{ (S,h) :\Big| \err(h; \DD)-\emperr\left(h ; S\right)
	\Big|>\e 
	\right\} .
	\]
	
	Using Chernoff's inequality, for each $h$,
	$$\mu_S\left(E_{h}\right)\ge2\d\exp\left(-2m\e^{2}\right),$$
	where $E_h$ is the fiber of $E$ over function $h$. 
	
	Lemma~\ref{lem:lemmaDiv2} implies
	\begin{align*}
	\mu \left(E\right)\leq   \frac{I(S;\cA(S))+1}{2m\e^{2}-1} ,
	\end{align*}
\end{proof}

We now prove lemmas~\ref{lem:lemmaDiv} and~\ref{lem:lemmaDiv2}.

\begin{proof}[Proof of Lemma~\ref{lem:lemmaDiv}]
	
	\begin{align*}
	\kl\left(\m||\n\right)  
	&  =  -\m\left(E\right)\underset{x\in E}{\sum}\frac{\m\left(x\right)}{\m\left(E\right)}\d\log\left(\frac{\n\left(x\right)}{\m\left(x\right)}\right)-\m\left(E^{c}\right)\underset{x\in E^{c}}{\sum}\frac{\m\left(x\right)}{\m\left(E^{c}\right)}\d\log\left(\frac{\n\left(x\right)}{\m\left(x\right)}\right)\\
	&  \stackrel{\left(a\right)}{\se}-\m\left(E\right)\d\log\left(\underset{x\in E}{\sum}\frac{\m\left(x\right)}{\m\left(E\right)}\frac{\n\left(x\right)}{\m\left(x\right)}\right)-\m\left(E^{c}\right)\d\log\left(\underset{x\in E^{c}}{\sum}\frac{\m\left(x\right)}{\m\left(E^{c}\right)}\frac{\n\left(x\right)}{\m\left(x\right)}\right) \\
	&  = -\m\left(E\right)\d\log\left(\frac{\n\left(E\right)}{\m\left(E\right)}\right)-\m\left(E^{c}\right)\d\log\left(\frac{\n\left(E^{c}\right)}{\m\left(E^{c}\right)}\right) \\
	& \stackrel{\left(b\right)}{\se} -\m\left(E\right)\d\log\left(\n\left(E\right)\right)-\m\left(E^{c}\right)\d\log\left(\n\left(E^{c}\right)\right)-1 \\
	& \geq -\m\left(E\right)\d\log\left(\n\left(E\right)\right)-1 ,
	\end{align*}
	where (a) follows by convexity, and (b) holds since the binary entropy is at most one. 
\end{proof}

\begin{proof}[Proof of Lemma~\ref{lem:lemmaDiv2}]
	
	By Lemma~\ref{lem:lemmaDiv}, for each $y$,
	$$\mu_{X|Y=y}(E_y) \leq \frac{\kl\left(\mu_{X|Y=y}||\mu _X\right)+1}{\log\left(1/\mu_X\left(E_{y}\right)\right)} \leq \frac{\kl\left(\mu_{X|Y=y}||\mu _X\right)+1}{\log\left(1/\alpha \right)}.$$
	Taking expectation over $y$ yields
	$$\mu(E) \leq  \frac{I(X;Y)+1}{\log\left(1/\alpha \right)}.$$
\end{proof}

\subsection{The Sample Complexity Bound is Sharp}\label{sec:sharp}

Standard bounds on the sample complexity of learning hypotheses classes of VC dimension $d$ imply that to achieve a fixed confidence $\delta \leq 1/2$ one must use at least 
\[
m = \Omega\left(\frac{d + \log(1/\delta)}{\eps^2}\right)
\]
examples in the non-realizeable case \citep[see e.g.,][theorem 6.7]{SS-text}, and this bound is sharp.

In contrast, Theorem~\ref{thm:info-comp} above states that achieving confidence $\delta$ requires 
\[
m = \Omega\left(\frac{d}{\eps^2}\cdot\frac{1}{\delta}\right)
\]
examples, where in this case $d$ is the bound on $I(\cA(S);S)$. A natural question to ask is whether this sample complexity bound is also sharp. We now show that indeed it is.

To see that the bound is tight for $d$ and $\eps$, consider the case where $\cX=[d]$ and $\hh=\{0,1\}^\cX$. For any learner $\cA$ it holds that
\[
I(\cA(S);S)\leq H(\cA(S)) \leq \log|\hh|=d
\]
However, the VC dimension of $\hh$ is also $d$. Because the bound for VC dimension is always sharp and it equals the bound from Theorem~\ref{thm:info-comp}, it follows that that bound is also sharp in this case.

To see that the bound is sharp in $\delta$ as well, consider the following proposition.

\medskip

\begin{prop}
	\label{prop:Sharpness}
	Let $n\geq m\geq 4$ be integers such that $n$ is sufficiently large.
	Let $\xx = [n]$ and let $\DD$ be the uniform distribution on examples of the form $\{(x,1) : x\in [n]\}$.
	There is a deterministic learning algorithm $\cA : (\cX \times \{0,1\})^m \to \{0,1\}^\cX$
	with sample size $m$ and mutual information $O(1)$
	so that
	\[ \pr{S}{\lvert \emperr\left(\cA(S); S\right)-\err(\cA(S); \DD)\rvert \geq 1/2} \geq \frac{1}{m}\]
	where $S$ is generated i.i.d.~from $\DD$ and $f$.
\end{prop}

The construction is based on the following claim.

\medskip

\begin{claim}
	\label{clm:setsTi}
	For a sufficiently large $n$, there are $M  = \Theta \left( \frac{ 2^m}{ m} \right)$ 
	subsets $T_1,\ldots,T_M$ of $[n]$ each of size at most $n/2$
	so that the $\DD$-measure of $T = \bigcup_i (T_i\times\{1\})^m$ is between
	$1/m$ and $4/m$.
\end{claim}

\begin{proof}
	Let $T_1,\ldots,T_M$ be i.i.d.~uniformly random 
	subsets of $\cX$, each of size $k = \lfloor n/2 \rfloor$.
	The size of each $T_i^m$ is $k^m$.
	For $i \neq j$ we have
	\begin{align*}
	\ex{}{|T_i^m \cap T_j^m|} 
	& = \ex{}{|T_i \cap T_j|^m} \\
	& = \sum_{x_1,\ldots,x_m} \ex{}{
		1_{x_1 ,\ldots,x_m \in T_i}}\ex{}{1_{x_1 ,\ldots,x_m \in T_j}} \\
	& \leq n^m \cdot\left( \frac{{n-m \choose k-m}}{{n \choose k}} \right)^2 + m^2 n^{m-1}\cdot 1,
	\end{align*}
	where in the last inequality, the first term corresponds to sequences $x_1,\ldots, x_m$ where all elements are distinct,
	and the second term corresponds to sequences $x_1,\ldots, x_m$ for which there are $i\neq j$ such that~$x_i = x_j$.
	Now, since
	$\frac{{n-m \choose k-m}}{{n \choose k}}
	= \frac{k(k-1)\cdots (k-m+1)}{n(n-1)\cdots (n-m+1)} 
	\leq \left(\frac{k}{n} \right)^m$,
	\begin{align*}
	\ex{}{|T_i^m \cap T_j^m|} 
	& \leq  \frac{k^{2m}}{n^m}  + m^2n^{m-1}.
	\end{align*}
	Therefore, since $\sum \lvert T_i^m\rvert \geq\lvert \cup_i T_i^m \rvert \geq \sum \lvert T_i^m\rvert - \sum_{i\neq j}\lvert T_i^m\cap T_j^m\rvert$,
	\begin{align*}
	M k^m\geq 
	\ex{}{ | T |}
	& \geq M k^m - \frac{M^2}{2}  \frac{k^{2m}}{n^m}
	- \frac{M^2}{2} m^{2} n^{m-1}  
	\geq \frac{M k^m}{2} - \frac{M^2}{2} m^{2} n^{m-1} ,
	\end{align*}
	as long as
	$M \leq n^m/k^m$.
	Hence, plugging $M  =  \frac{ 4n^m}{ m k^m}=\Theta(\frac{2^m}{m}) $
	yields that 
	\[\frac{4}{m} \geq \ex{}{\lvert T\rvert/ n^m} \geq \frac{2}{m}-o(1),\]
	where the $o(1)$ term approaches 0 as $n$ approaches $\infty$.
	So, for a sufficiently large $n$ there is a choice of $T_1,\ldots,T_M$ as claimed.
\end{proof}

\begin{proof}[Proof of Proposition~\ref{prop:Sharpness}]
	The algorithm $\cA$ is defined as follows.
	Let $T_1,\ldots ,T_M$ be the sets given in Claim~\ref{clm:setsTi}.
	For each $i \in [M]$, let $h_i$ be the hypothesis that is $1$
	on $T_i$ and $0$ elsewhere.
	Given a sample $((x_1,1),\ldots,(x_m,1))$,
	the algorithm outputs $h_i$, where
	$i$ is the minimum index so that $\{x_1,\ldots,x_m\} \subset T_i$;
	if no such index exists the algorithm outputs the all-ones function.
	
	The empirical error of the algorithm is $0$,
	and with probability $p \in [1/m,4/m]$
	it outputs a hypothesis with true error at least $1/2$.
	The amount of information it provides on its inputs can be bounded
	as follows: letting $p_i$ be the probability that the algorithm
	outputs $h_i$, we have
	\begin{align*}
	I(S;\cA(S)) 
	& = H(\cA(S)) \\
	& = (1-p) \log \frac{1}{1-p}
	+ p \sum_i \frac{p_i}{p} \log \frac{1}{p_i} \\
	& \leq (1-p) \log \frac{1}{1-p}
	+ p \log \frac{M}{p} \tag{convexity} \\
	& \leq O(1).
	\end{align*}
\end{proof}

\section{A Lower Bound on Information}\label{sec:lower}

In this section we show that any proper consistent learner for the class of thresholds cannot use only little information with respect to all realizable distributions $\DD$. Namely, we find for every such algorithm $\cA$ a realizable distribution $\DD$ so that $I(S;\cA(S))$ is large.

Let $\mathcal{X} = [2^n]$ and let $\mathcal{T}\subseteq \{0,1\}^\mathcal{X}$ be the set of all thresholds; that is $\mathcal{T} = \{f_k\}_{k \in [2^n]}$ where
\[
f_k(x) = 
\left\{
\begin{array}{ll}
0  & x < k ,\\
1 & x \geq k .
\end{array}
\right.
\]

\medskip

\begin{theorem}
	\label{thm:lower_bound}
	For any consistent and proper learning algorithm $\cA$ for $\mathcal{T}$ with sample size $m$ 
	there exists a {realizable} distribution $\DD = \DD(\cA)$ so that {
		\[
		I(S;\cA(S)) = \Omega \left( \frac{\log n}{m^2} \right) = \Omega \left( \frac{\log \log |\mathcal{X}|}{m^2} \right) ,
		\]
		where $S\sim \DD^m$}.
\end{theorem}

The high-level approach is to identify in $\cA$ a rich enough structure and use it to define the distribution $\DD$. Part of the difficulty in implementing this approach is that we need to argue on a general algorithm, with no specific structure.
A different aspect of the difficulty in defining $\DD$ stems from that we can not adaptively construct $\DD$, we must choose it and then the algorithm gets to see many samples from it.

\subsection{Warm Up}

We first prove Theorem~\ref{thm:lower_bound} for the special case of deterministic learning algorithms. Let $\cA$ be a consistent deterministic learning algorithm for $\mathcal{T}$. Define the $2^n \times 2^n$ upper triangular matrix $M$ as follows. For all $i<j$, $$M_{ij}=k$$ where $f_k$ is the output of $\cA$ on a sample of the form 
\[
S_{ij} = \Big(\underbrace{(1,0), \dots, (1,0)}_{\text{ m-2 }}, (i,0), (j,1) \Big)
\]
and $M_{ij}=0$ for all $i \geq j$. 

The matrix $M$ summarizes the behavior of $\cA$ on some of its inputs. Our goal is to identify a sub-structure in $M$,
and then use it to define the distribution $\DD$.

We start with the following lemma.

\medskip

\begin{lem}\label{deteministic_induction_lemma}
	Let $Q \in\mathrm{Mat}_{2^{n} \times 2^{n}} (\mathbb{N})$ be a symmetric matrix that has the property that for all $i,j$:
	\[
	\min\{i,j\} \leq Q_{ij} \leq \max\{i,j\} \tag{i}
	\]
	Then $Q$ contains a row with at least $n+1$ different values (and hence also a column with $n+1$ different values).
\end{lem}

\begin{proof} The proof is by induction on $n$. In the base case $n=1$ we have
	$$Q = \begin{bmatrix}
	1 & x \\
	x & 2 
	\end{bmatrix}$$ and the lemma indeed holds. 
	
	For the induction step, 
	let \[
	Q = \begin{bmatrix}
	Q_1 & Q_2 \\
	Q_3 & Q_4 
	\end{bmatrix} \in \mathrm{Mat}_{2^{n+1} \times 2^{n+1}} (\mathbb{N})
	\]
	where $Q_1,Q_2,Q_3,Q_4 \in \mathrm{Mat}_{2^{n} \times 2^{n}} (\mathbb{N})$. 
	
	All the values in $Q_1$ are in the interval $[1,2^n]$ and $Q_1$ is also a symmetric matrix and satisfies property (i). So $Q_1$ contains some row $r$ with at least $n+1$ distinct values in the interval $[1,2^n]$. 
	
	Similarly, all the values in $Q_4$ are in the interval $[2^n+1, 2^{n+1}]$, and we can write $Q_4$ as
	\[
	Q_4 = Q_4' + 2^n\cdot J
	\]
	where $J$ is the all-1 matrix and $Q_4'=Q_4-(2^n\cdot J)$ is a symmetric matrix satisfying property (i). From the induction hypothesis it follows that $Q_4$ contains a column $k$ with at least $n+1$ different values in the interval $[2^n+1, 2^{n+1}]$.
	
	Now consider the value $Q_{r k}$. If $(Q_2)_{r k} \in [1, 2^n]$ then the column in $Q$ corresponding to $k$ contains $n+2$ different values. Otherwise, the row corresponding to $r$ contains $n+2$ different values. 
\end{proof}

Next, consider the matrix
\[
Q = M + M^t+\mathrm{diag}(1,2,3,\dots,2^n) .
\]
It is a symmetric matrix satisfying property (i), and so it contains a row $r$ with at least $n+1$ distinct values. If row $r$ contains at least $\frac{n}{2}$ distinct values above the diagonal 
then row $r$ in $M$ also contains $\frac{n}{2}$ distinct values above the diagonal. Otherwise, row $r$ contains at least $\frac{n}{2}$ distinct values below the diagonal, and then column $r$ in $M$ contains $\frac{n}{2}$ distinct values above the diagonal.

The proof proceeds by separately considering each of these two cases (row or column):

{\bf Case 1: $M$ contains a row $r$ with $\frac{n}{2}$ distinct values.} Let $k_1, \dots, k_{n/2}$ be columns such that all the values $M_{rk_i}$ are distinct. Let $\DD$ be the distribution that gives the point $1$ a mass of $1-\frac{1}{m-2}$, gives $r$ a mass of $\frac{1}{2(m-2)}$ and evenly distributes the remaining mass on the values $k_1,\ldots,k_{n/2}$. The function labeling the examples is chosen to be $f_{r+1} \in {\cal T}$.

Let $E$ be the indicator random variable of the event that the sample $S$
is of the form
\[
S  = \Big(\underbrace{(1,0), \dots, (1,0)}_{\text{ m-2 }}, (r,0), (k_i,1) \Big)
\]
for some $k_i$. The probability of $E=1$ is at least
\[
\Big(1-\frac{1}{m-2}\Big)^{m-2}\cdot \Big(\frac{1}{2(m-2)}\Big)^2  =
\Omega \left( \frac{1}{m^2} \right).
\]
From the definition of $M$, when $E=1$  the algorithm outputs $h=f_{M_{r k_i}}$ where $k_i$ is uniformly distributed uniformly over $\frac{n}{2}$ values, and so
$H(\cA(S)|E=1) = \log(n/2)$.
This yields 
\begin{align*}
I(S;\cA(S)) 
& = I(S,E;\cA(S)) \\ 
& \geq I(S;\cA(S)|E) \\ 
& = H(\cA(S)|E) \\
& \geq \pr{}{E=1}H(\cA(S)|E=1) \\
& \geq \Omega \left( \frac{ \log n}{m^2} \right).
\end{align*}

{\bf Case 2: $M$ contains a column $k$ with $\frac{n}{2}$ distinct values.} Let $r_1, ..., r_{n/2}$ be rows such that all the values $M_{r_ik}$ are distinct. Now, consider the event that
\[
S = \Big(\underbrace{(1,0), \dots, (1,0)}_{\text{ m-2 }}, (r_i,0), (k,1) \Big)
\]
for some $r_i$. The rest of the argument is the same as for the previous case.

\subsection{Framework for Lower Bounding Mutual Information}

Here we describe a simple framework that allows to lower bound the mutual information between two random variables.

\subsubsection*{Standard Lemmas\footnote{We include the proofs for completeness.}}

\begin{lem}\label{lower_bound_on_negative_mutual_information}
	For any two distribution $p$ and $q$, the contribution of the terms with $p(x)< q(x)$ to the divergence is at least $-1$:
	\[
	\sum_{x:p(x)< q(x)} p(x)\log \frac{p(x)}{q(x)} > -1.
	\]
\end{lem}

\begin{proof}
	Let $E$ denote the subset of $x$'s for which $p(x) < q(x)$. Then we have
	\begin{align*}
	\sum_{x \in E} p(x)\log \frac{p(x)}{q(x)} 
	& \geq -p(E) \cdot \sum_{x \in E} p(x|E)\log \frac{q(x)}{p(x)} \\
	& \geq -p(E) \cdot \log \sum_{x \in E} p(x|E) \frac{q(x)}{p(x)}  \\
	& = -p(E) \cdot \log  \frac{q(E)}{p(E)} \\
	& \geq p(E) \cdot \log p(E). 
	\end{align*}
	For $0 \leq z \leq 1$, $z \log z$ is maximized when its derivative is $0$: $\log e + \log x = 0$. So the maximum is attained at $z = 1/e$, proving that $p(E) \log p(E) \geq \frac{- \log e}{e} > -1$.
\end{proof}

\medskip

\begin{lem}[Data processing]
	\label{lemma_conditional_independence_implies_lower_mutual_info}
	Let $X,Y,Z$ be random variables such that $X-Y-Z$ form a Markov chain;
	that is, $X$ and $Z$ are independent conditioned on $Y$. Then
	\[
	I(X;Y) \geq I(X;Z)
	\]
\end{lem}

\begin{proof}
	The chain rule for mutual information yields that
	\begin{align*}
	I(X;Y) 
	& = I(X;Z) + I(X;Y|Z) - I(X;Z|Y) \tag{chain rule} \\
	& = I(X;Z) + I(X;Y|Z) \tag{$X-Y-Z$} \\
	& \geq I(X;Z) \tag{information is non-negative}.
	\end{align*}
\end{proof}

\subsubsection*{The Framework}

The following lemma is the key tool in proving a lower bound on the mutual information.

\medskip

\begin{lem}\label{lemma_probabilities}
	Let $n\in\mathbb{N}$, and let $p_1, \dots,p_n$ be probability distributions over the set $[n]$ such that for all $i \in [n]$,
	\[
	p_i(i)\geq \frac{1}{2} \tag{iv}
	\]
	Let $U$ be a random variable distributed uniformly over $[n]$. Let $T$ be
	a random variable over $[n]$ that results from sampling an index $i$ according to $U$ and then sampling an element of $[n]$ according to $p_i$.
	Then 
	\[
	I(U; T) = \Omega(\log n).
	\]
\end{lem}

\begin{proof}
	\begin{align*}
	I(U; T)
	& =\sum_{i=1}^n \sum_{t=1}^n p_U(i) p_i(t)\log\frac{p_i(t)}{p_T(t)} \\
	& =\sum_{i} p_U(i) p_i(i)\log\frac{p_i(i)}{p_T(i)}+\sum_{i} \sum_{t\neq i} p_U(i) p_i(t)\log\frac{p_i(t)}{p_T(t)} .
	\end{align*}
	Consider the first sum (the ``diagonal"):
	\begin{align*}
	\sum_{i} p_U(i) p_i(i)\log\frac{p_i(i)}{p_T(i)} 
	& = \frac{1}{n}\sum_{i} p_i(i)\log\frac{p_i(i)}{p_T(i)} \\
	& \geq \frac{1}{n}\sum_{i} \frac{1}{2}\log\frac{\frac{1}{2}}{p_T(i)}  \tag{iv} \\
	& \geq \frac{1}{n} \cdot \frac{n}{2}\log\frac{\frac{n}{2}}{1} \tag{log-sum inequality} \\
	& = \frac{1}{2} \log \frac{n}{2} .
	\end{align*}
	
	Finally, Lemma \ref{lower_bound_on_negative_mutual_information} 
	implies that the second sum (the ``off-diagonal") 
	is at least $-1$. 
\end{proof}

We generalize the previous lemma as follows.

\medskip

\begin{lem}\label{lemma_probabilities_extended}
	Let $p_1, \dots,p_n$ be probability distributions over $\cX$.
	Let $S_1, \dots, S_n \subset \cX$ be pairwise disjoint events such that
	for all $i \in [n]$,
	\[ p_i(S_i)\geq \frac{1}{2} \tag{v}
	\]
	Let $U$ be a random variable distributed uniformly over $[n]$. Let $W$ be
	a random variable taking values in $\cX$ that results from sampling an index $i$ according to $U$ and then sampling an element of $\cX$ according to $p_i$.
	Then,
	\[
	I(U; W) = \Omega(\log n).
	\]
\end{lem}

\begin{proof}
	Let $S_0 = \cX \setminus (S_1 \cup \cdots \cup S_n)$, and let 
	$T$ be the random variable taking values in $\NN$
	defined by $T(W) = i$ iff $W \in A_i$.
	Hence,
	$T$ satisfies the conditions of Lemma~\ref{lemma_probabilities}. Furthermore, we have that the random variables $U,W,T$ form the following Markov chain:
	$U-W-T$.
	We thus conclude that
	\[
	I(U;W) \geq I(U;T) = \Omega(\log n)
	\]
	where the inequality is according to the data processing inequality
	(Lemma~\ref{lemma_conditional_independence_implies_lower_mutual_info}) 
	and the equality follows from Lemma~\ref{lemma_probabilities}. 
\end{proof}

\subsection{Proof for General Case}

We start with the analog of Lemma~\ref{deteministic_induction_lemma} from the warm up.
Let $\Delta([2^n])$ be the set of all probability distribution over $[2^n]$. Let $Q\in\mathrm{Mat}_{2^{n} \times 2^{n}} (\Delta([2^n]))$, i.e., $Q$ is a $2^{n} \times 2^{n}$ matrix where each cell contains a probability distribution.

\medskip

\begin{lem}\label{non_deteministic_induction_lemma}
	Assume that $Q$ is symmetric and that it has the property that for all~$i,j$,
	\[
	\mathrm{supp}( Q_{ij}) \subseteq [\min\{i,j\}, \max\{i,j\}] \tag{ii} .
	\]
	Then, $Q$ contains a row with $n+1$ distributions $p_1, \dots, p_{n+1}$ such that there exist pairwise disjoint sets $S_1, \dots,S_{n+1} \subset [2^n]$ so that for all $i \in [n+1]$,
	\[
	p_i(S_i) \geq \frac{1}{2} \tag{iii}
	\]
	(and hence it also contains such a column).
\end{lem}

\begin{proof}
	Again, the proof is by induction on $n$. 
	The base case is easily verified.
	For the step,
	let \[
	Q = \begin{bmatrix}
	Q_1 & Q_2 \\
	Q_3 & Q_4 
	\end{bmatrix}
	\]
	where $Q_1,Q_2,Q_3,Q_4\in\mathrm{Mat}_{2^n\times2^n}(\Delta([2^n]))$. 
	
	The matrix $Q_1$ is symmetric, it satisfies property (ii), and that the supports of all the entries in $Q_1$ are contained in $[2^n]$. So by induction $Q_1$ contains a row $r$ with probability functions $p_1,\dots,p_{n+1}$ and pairwise disjoint sets $A_1,\dots,A_{n+1} \subset [2^n]$ that satisfy property (iii).
	
	Similarly, 
	one sees that $Q_4$ satisfies property (iii) as well. Namely, $Q_4$ contains a column $k$ with probabilities $q_1,\dots,q_{n+1}$ and pairwise disjoint sets $B_1,\dots,B_{n+1} \subset [2^n+1,2^{n+1}]$ with $q_i(B_i)\geq \frac{1}{2}$ for all $i$.
	
	We now consider the probability distribution $Q_{rk}$. If $Q_{rk}([2^n])\geq \frac{1}{2}$ then we define $q_{n+2}=Q_{rk}$ and $B_{n+2}=[2^n]$. Thus, column $k$ of $Q$ satisfies property (iii) with probabilities $q_1,\dots,q_{n+2}$ and sets $B_1,\dots,B_{n+2}$. Otherwise, we choose $p_{n+2}=Q_{rk}$ and $A_{n+2}=[2^n+1,2^{n+1}]$. In this case, the row of $Q$ corresponding to $r$ satisfies property (iii) with probabilities $p_1,\dots,p_{n+2}$ and sets $A_1,\dots,A_{n+2}$. 
\end{proof}

In the previous section we proved that property (iii) yields a lower bound on mutual information.
We are prepared to prove the desired lower bound for probabilistic algorithms.

\begin{proof}[Proof of Theorem~\ref{thm:lower_bound}]
	Let $\cA$ be a consistent and proper learning algorithm for $\mathcal{T}$.
	Let $M \in \mathrm{Mat}_{2^n\times2^n}(\Delta([2^n]))$ be the matrix
	whose $(i,j)$ entry for $i \neq j$ is
	
	the probability distribution $\cA$ induces on ${\cal T}$ with input
	
	\[
	\Big(\underbrace{(1,0), \dots, (1,0)}_{\text{ m-2 }}, (\min\{i,j\},0), (\max\{i,j\},1) \Big)
	\]
	and whose $(i,i)$ entry for all $i$ is the degenerate distribution that assigns probability 1 to $i$.
	The matrix $M$ is symmetric, and because $\cA$ is consistent and proper it follows that $M$ satisfies property (ii). By Lemma~\ref{non_deteministic_induction_lemma} therefore $M$ contains a row $r$ with probabilities $p_1,\dots,p_{n+1}$ for which there are pairwise disjoint sets $S_1,\dots,S_{n+1} \subseteq [2^n]$ such that $p_i(S_i) \geq \frac{1}{2}$ for all $i$. 
	
	We assume without loss of generality that the probabilities $p_1,\dots,p_{n/2}$ are located on row $r$ above the diagonal in cells $(r,k_1),\dots,(r,k_{n/2})$; the symmetric case can be handled similarly.
	
	We now construct the probability $\DD$ over $\cX$.
	The probability of $1$ is $1-\frac{1}{m-2}$, the probability of $r$ is $\frac{1}{2(m-2)}$ and rest of the mass
	is uniformly distributed on $k_1,\ldots,k_{n/2}$.\footnote{We assume for simplicity that $1$, $r$ and the $k_i$ are all distinct.} Consider the indicator random variable $E$ of the event that $S$ is of the form
	\[
	S = \Big(\underbrace{(1,0), \dots, (1,0)}_{\text{ m-2 }}, (r,0), (k_i,1) \Big)
	\]
	for some $k_i$. For every $S$ in this event, denote by $i_S$ the index so that
	$k_{i_S}$ is in the last example of $S$. Again, $\Pr[E=1] \geq \Omega(1/m^2)$. Finally, as in the warm up,
	\begin{align*}
	I(S;\cA(S)) 
	& \geq I(S;\cA(S)|E) \\
	& \geq \Pr[E=1]\cdot I(S;\cA(S)|E=1) \\
	& \geq \Pr[E=1]\cdot I(i_S;\cA(S)|E=1) \\
	& \stackrel{(*)}{\geq} \Omega \left(\frac{\log n}{m^2} \right) ,
	\end{align*}
	where $(*)$ is justified as follows:
	Given that $E=1$, we know that $i_S$ is uniformly distributed on $[n/2]$. Furthermore, $\cA(S)$ is the result of sampling a hypothesis according to the distribution $p_{i_S}$.
	The lower bound hence follows from Lemma~\ref{lemma_probabilities_extended}.
	%
\end{proof}

\section{Approaches for Proving Upper Bounds on Information}\label{sec:upper-approaches}

Here we give a simple and generic information learner 
that provides sharp upper bounds for some basic concept classes. 
For example, we get an ERM that is a $O(\log\log N)$-bit information learner for the class of thresholds over $[N]$; this bound is indeed tight given the lower bound in Section~\ref{sec:lower}. 

\medskip

\begin{definition}[Generic information learner]
	The generic algorithm $\cA_{\hh}$ for a {hypothesis} class $\hh$ acts as follows.
	Given a {realizable} sample $S$, the algorithm first 
	finds the set of all hypotheses in $\hh$ that are consistent with $S$, and
	then it simply outputs a uniformly random hypothesis from that set.
\end{definition}

This algorithm is well defined for finite $\hh$ and it is proper and consistent.
Specifically, if $\hh$ has finite VC-dimension then $\cA_\hh$ PAC-learns the class with the standard sample complexity of VC classes (when there are no information constraints).

\subsection{A Method for Upper Bounding Information}\label{sec:upper-method}

The following simple lemma provides a useful method for upper bounding mutual information.
Let $\hh\subseteq \{0,1\}^{\xx}$ and $m\in \NN$. Let $\cA: \left(\xx\times\{0,1\}\right)^m\rightarrow \hh$ be an algorithm that takes a sample $S$ of $m$ examples and outputs a hypothesis $h = \cA(S) \in\hh$.
Recall that $P_{h|S}$ denotes the distribution of $\cA(S)$ when the input sample is $S$,
and that
$P_{h}(\cdot)=\ex{S\sim\DD^m}{P_{h|S}(\cdot)}$ denotes the marginal distribution of the output of $\cA$ with respect to some fixed input distribution $\DD$.

\medskip

\begin{lem}\label{lem:upper_bd_mut_inf}
	For every distribution $\qq$ over $\hh$, we have 
	$$I(S;\cA(S)) = \ex{S\sim \DD^m}{\kl\left(P_{h|S} ~ || ~P_{h}\right)}\leq \max\limits_{S\in\supp(\DD^m)}\kl\left(P_{h|S}~ ||~ \qq \right).$$
\end{lem}

\begin{proof}
	Observe that for any distribution $\qq$ over $\hh$,
	\begin{align}
	I(S;\cA(S))  & = \ex{S\sim \DD^m}{\kl\left(P_{h|S} ~ || ~\qq\right)} - \kl\left(P_{h} ~ || ~\qq\right) \nonumber \\
	&\leq \ex{S\sim \DD^m}{\kl\left(P_{h|S} ~ || ~\qq \right)}\nonumber\\
	&\leq \max\limits_{S\in\supp(\DD^m)}\kl\left(P_{h|S}~ ||~ \qq \right)
	\nonumber  
	\end{align}
	(if the KL-divergence is infinite, the upper bound trivially follows).
\end{proof}

\subsection{Examples}

We demonstrate the behavior of the generic algorithm
in a couple of simple cases.

\subsubsection*{Thresholds}
Consider the class of thresholds ${\cal T}$
over $\xx = [N]$ and the generic algorithm $\cA_{\cal T}$.
\medskip

\begin{theorem}\label{thm:thresh}
	For every sample size,
	the generic algorithm $\cA = \cA_{\cal T}$ is a $(\log\log(N)+O(1))$-bit information learner.
\end{theorem}

\begin{proof}
	Suppose the labelling function is $f_k(x) = 1_{x \geq k}$.
	Consider the following distribution $\qq$ over {$\hh$}:
	\begin{align}
	\qq({t})&=\frac{c}{(1+|k-{t}|)\log (N)}, \nonumber
	\end{align}
	where $c \geq 1/2$ is the normalizing constant. One can verify this lower bound on $c$ by noting that 
	$$\sum_{t\in [N]}\frac{1}{1+|k-t|}\leq \int_{t=1}^{N+1}\frac{1}{1+|k-t|} dt \leq 2\log(N).$$
	Now, by plugging this choice of $\qq$ into Lemma~\ref{lem:upper_bd_mut_inf} and noting that for any realizable sample $S$, 
	the distribution $P(h | S)$ is uniform over ${f_t}$ for
	$t \in \{x_1, x_1+1, \ldots, x_2\}$ for some $x_1 \leq k \leq x_2$ in $\xx$, we can reach the desired bound:
	\begin{align*}
	\kl\left(P_{h|S}~ ||~ \qq \right)
	& = \log \log (N) + \sum_{x_1 \leq t \leq x_2} \frac{1}{x_2-x_1+1} \log \frac{1+|k-t|}{{c(x_2-x_1+1)}} \\
	& \leq \log \log (N) + { 1},
	\end{align*}
	since $\frac{1+|k-t|}{c(x_2-x_1+1)}\leq 2$.	
\end{proof}

Given our results in Section~\ref{sec:lower}, we note that the above bound is indeed tight.

\subsubsection*{Point Functions} \label{sec:point_fn}

We can learn the class of point functions ${\cal PF}$ on $\xx=[N]$ with at most $2$ bits of information. The learning algorithm is again the generic one. 
\medskip

\begin{theorem}
	For every realizable distribution $\DD$, and for every sample size $m \leq N/2$, the generic algorithm $\cA = \cA_{\cal PF}$ has at {most $2$ bits} of mutual information with respect to $\DD$.
\end{theorem}

\begin{proof}
	Suppose, without loss of generality, that the target concept is $f_1(x) = 1_{\{x=1\}}$. 
	{Pick} $\qq$ in the bound of Lemma~\ref{lem:upper_bd_mut_inf} as follows:
	\begin{align}
	\qq(x)&= \begin{cases} 1/2 &  x=1 ,\\ 
	\frac{1}{2(N-1)} & x \neq 1 .
	\end{cases}\nonumber
	\end{align}
	Let $S$ be a sample.
	If $1$ appears in $S$ then 
	$$\kl\left(P_{h|S}~ ||~ \qq \right)
	= 1\cdot\log \frac{1}{1/2} = 1.$$
	If $1$ does not appear in $S$ then $P_{h|S}$
	is uniform on a subset of the form $\{ 1_{x=i}: i \in \xx'\}$ 
	for some $\xx' \subset \xx$ that contains $1$ of size $k \geq N-m\geq N/2$,
	so
	\begin{align*}
	\kl\bigl(P_{h|S}~ ||~ \qq \bigr)
	& = \frac{1}{k} \log \frac{2}{k}
	+   \frac{k-1}{k} \log \frac{2(N-1)}{k}{\leq 2.}
	\end{align*}
	
\end{proof}

\subsubsection*{Separtion between Differential Privacy and Bounded Information}

The above result, together with known properties of VC classes and \cite[Corollary 1]{BeiNS10}, implies a separation between the family of proper $d$-bit information learners and the family of pure differentially private proper learners. 

\medskip
{
	\begin{cor}\label{cor:privacy-separation}
		There is a proper $2$-bit information learner for the class of point functions over $[N]$ with
		sample size $O\left(\frac{\log(1/\delta)}{\eps}\right)$. On the other hand, for any $\alpha>0$, any $\alpha$-differentially private algorithm that \emph{properly} learns point functions over $[N]$ requires $\Omega\left(\frac{\log N+\log(1/\delta)}{\eps\alpha}\right)$ examples.
	\end{cor}
}

\subsection{The Generic Learner is Far from Optimal}

We have seen that the generic algorithm provides sharp upper bounds on the mutual information for some cases. However, there are also some simple settings in which it reveals a lot more information than is necessary. Take the following class
\[
\hh=\{ 1_{x=i} : 1<i\leq N  \} \cup \{1_{x>1} \}
\]
over the space $\cX=[N]$ and the distribution with $P(x=1)=1-\frac{1}{m}$ and $P(x=i)=\frac{1}{m(N-1)}$ for $i\neq 1$. Given a sample $S$ of size $m$ labelled by the function $1_{x>1}$, we calculate the mutual information of $\cA=\cA_{\hh}$ in this setting
(we think of $m$ as a large constant).
We start with some preliminary calculations.

Conditional entropy:
Let $s_1=((1,0),(1,0),...,(1,0))$. Then
\begin{align*}
H(\cA(S)|S) & = P(S=s_1)H(\cA(S)|S=s_1)+\sum_{s\neq s_1} P(S=s)H(\cA(S)|S=s) \\
& <    \left(1-\frac{1}{m}\right)^m \log N +1   <   \frac{\log N}{e}  + 1 .
\end{align*}
Marginal probabilities:
\[
P(\cA(S)=1_{x>0})= \left(1-\frac{1}{m}\right)^m \frac{1}{N}+m\left(1-\frac{1}{m}\right)^{m-1}\cdot  \frac{1}{m} \cdot  \frac{1}{2}+...  > \frac{0.99}{2e} 
\]
and
\[
P(\cA(S)=1_{x=i})= \left(1-\frac{1}{m}\right)^m \frac{1}{N}+m\left(1-\frac{1}{m}\right)^{m-1}\cdot  \frac{1}{m(N-1)} \cdot  \frac{1}{2}+... <  \frac{3.01}{2eN}.
\]
Entropy:
\begin{align*}
& H(\cA(S)) \\
& =-P(\cA(S)=1_{x>0})\log P(\cA(S)=1_{x>0}) - \sum_i  P(\cA(S)=1_{x=i})\log P(\cA(S)=1_{x=i}) \\
& > \frac{3}{2e}\log N .
\end{align*}
Finally, we can bound the information:
$$I(S;\cA(S))=H(\cA(S))-H(\cA(S)|S) > \frac{\log N}{2e} -1.$$

This calculation makes sense, since the only way to learn something substantial about the sample is when you get exactly one $x\neq 1$. This happens with probability $\approx 1/e$ and then 
the generic algorithm chooses a function that reveals $x$ with probability $1/2$. This sample has entropy $\log(N-1)$, so we have the desired result.

In comparison, let us define a deterministic ERM with low 
information for this class. Define $E$ to be the event where the sample is consistent with $1_{x>1}$. When $E$ occurs, the algorithm outputs $1_{x>1}$. Otherwise, the algorithm outputs $1_{x=k}$ where $k$ is the minimal integer such that $1_{x=k}$ is consistent with the sample.
This algorithm satisfies:
$$I(S;\cA(S)) = H(\cA(S))\leq H(\cA(S),1_E)=H(1_E) +H(\cA(S)|1_E)\leq 1 + \log (m+2).$$

\subsection{The Distribution-Dependent Setting}\label{sec:distribution-dependent}

{
	As was shown in Section~\ref{sec:lower},
	there are classes of VC dimension 1 (thresholds on a domain of size $N$)
	for which every proper consistent learner must leak at least
	$\Omega(\log\log N)$ bits of information on some realizable distributions.
	
	Here, we consider the distribution-dependent setting;
	that is, we assume the learner knows the marginal distribution $\DD_\xx$ on inputs 
	(but it does not know the target concept). 
	We show that in this setting every VC class 
	can be learned with relatively low information.}

\medskip

\begin {theorem} 
Given the size $m$ of the input sample, 
a distribution $\DD_\xx$ over $\xx$ and $\hh\con \{0,1\}^\xx$ with VC-dimension $d$, 
there exists a consistent, proper, and deterministic 
learner with $O(d \log (m+1))$-bits of information
(for $\hh$-realizable samples).
\end{theorem}

Before proving the theorem, we discuss a somewhat surprising phenomenon.
In a nutshell, the theorem says that for every distribution there is
a deterministic algorithm with small entropy.
It is tempting to ``conclude'' using von Neumann's minimax theorem
that this implies that there is a randomized algorithm
that for every distribution has small information.
This ``conclusion'' however is false,
as the threshold example shows.

\begin{proof}
	Let $\eps_k = (1/(m+1))^k$ for $k>0$.
	For each $k$, pick an $\eps_k$-net $N_k$ for $\hh$ 
	with respect to $\DD$ of minimum size;
	that is, for every $h \in \hh$ there is $f \in N_k$ so that {$\Pr_{x\sim \DD_\xx}\bigl(h(x)\neq f(x)\bigr) \leq \eps_k$}.
	A result of \cite{Hassler} states that the size
	of $N_k$ is at most $(4e^2/\eps_k)^d \leq (4e^2 m)^{kd}$.
	
	The algorithm works as follows:
	given an input sample $S$, the algorithm checks if $N_1$ contains a consistent hypothesis.
	If it does, the algorithm outputs it.
	Otherwise, it checks in $N_2$, and so forth.
	{As we explain below, the probability that the algorithm stops is one
		(even when $\xx$ is infinite).}
	Denote by $K$ the value of $k$ in which the algorithm stops.
	
	Bound the entropy of the output as follows:
	\[
	H(\cA (S)) \leq H(\cA (S), K) = H(\cA (S) |K) +H(K) \leq
	\]
	\[
	\leq \sum_k \Pr[K=k] \cdot k d \log (4e^2 (m+1)) + H(K)
	\]
	The labelling function $f \in \hh$ is at distance of at most $\eps_k$ 
	from $N_k$, so 
	$$\Pr[K \leq k] \geq (1-\eps_k)^m \geq 1- \frac{1}{(m+1)^{k-1}},$$
	which implies
	$\Pr[K = k+1] \leq \frac{1}{(m+1)^{k-1}} \leq \frac{1}{2^{k-1}}$
	({this in particular implies that the algorithm terminates with probability one}).
	Hence,
	$$\sum_k \Pr[K=k] \cdot k d \log (4e^2 (m+1)) \leq O(d \log (m+1))$$
	and $H(K) \leq O(1)$.
\end{proof}


\acks{AY was supported by ISF Grant No. 1162/15.}

\appendix

\section{Three additional proofs of Theorem \ref{thm:info-comp}}\label{additional-proofs}

\subsection{Proof II: De-correlating}\label{proof-ii}

The second proof we present allows to ``de-correlate'' two random variables in terms of the mutual information.
The lemma follows from \cite{SamRej} and \cite{log(I)}.

\medskip

\begin{lem}
	\label{lem:SectionLemma}
	Let $\mu$ be a probability distribution over $\cX\x \cY$, where $\cX$ and $\cY$ are finite sets.
	Let $(X,Y)$ be chosen according to $\mu$. 
	Then there exists a random variable $Z$ such that:
	\begin{enumerate}
		\item $Z$ is independent of $X$.
		\item $Y$ is a deterministic function of $(X,Z)$.
		\item $H\left(Y|Z \right) \leq I\left(X,Y \right)+\log (I(X,Y)+1)+O(1)$ .
	\end{enumerate}
	
\end{lem}

The lemma can be interpreted as follows. Think of $X$ as sampled from some unknown process,
and of $Y$ as the output of a randomized algorithm on input $X$. Think of $Z$ as the random coins of the algorithm. The lemma says that there is a way to sample $Z$, before seeing $X$, in a way that preserves functionality (namely, so that $(X,Y)$ are correctly distributed) and so that for an average $Z$, the randomized algorithm outputs only a small number of $Y$'s.

Theorem \ref{thm:info-comp} now follows similarly to the proof sketch for deterministic algorithms in section~\ref{sec:proof-sketch}, since conditioned on $Z$, the algorithm uses only a small number of outputs (on average).

\subsection{Proof III: Stability}\label{app:gen-stability}\label{proof-iii}

The third approach for proving the connection between generalization and bounded information is based on stability. The parameters obtained from this proof are slightly different from those given in the statement of Theorem~\ref{thm:info-comp}. For a $d$-bit information learner $\cA$, we prove that 
\begin{align}
\ex{\cA, S}{\err(\cA(S); \cD)-\emperr\left(\cA(S); S\right)}&<\sqrt{\frac{d}{m}} \label{stability-main}
\end{align}
where $S$ is the input sample of $m$ examples drawn i.i.d. from $\cD$.

We start by setting some notation. For random variables $X, Y$, we use the notation $\tvd(X, Y)$ to denote the total variation (i.e., the statistical distance) between the distributions of $X$ and $Y$. For clarity, we also twist the notation we used earlier a little bit, and use $\kl(X||Y)$ to denote the KL-divergence between the \emph{distributions} of $X$ and $Y$. Also, in the following, we use the notation $(X, Y)$ to denote the joint distribution of $X$ and $Y$, and $X\times Y$ to denote the product distribution resulting from the marginal distributions of $X$ and~$Y$. 

Finally, as typical in stability arguments, for any {sample} $S$ and any {example} $z\in\cX\times\{0,1\}$, we will use the notation $S^{(i, z)}$ to denote the set resulting from replacing the $i$-th example in $S$ by $z$. 

The proof relies on the following two lemmas. 

\medskip

\begin{lem}
	\label{lem:1stab}
	If $I\left(\cA\left(S\right); S\right)\leq d$, ~then $$\frac{1}{m}\sum_{i=1}^m\sqrt{I\left(\cA\left(S\right); S_i\right)}\leq \sqrt{\frac{d}{m}}$$ where $S_i$ denotes the $i$-th example in $S$. 
\end{lem}

\bigskip

\begin{lem}
	\label{lem:2stab}
	For any $i\in [m]$, we have 
	$$\sqrt{I\left(\cA\left(S\right); S_i\right)}\geq \ex{z}{\tvd\left(\cA\left(S^{(i, z)}\right), ~\cA\left(S\right)\right)}$$
	where $z = (x , y)\sim \DD$ independently from $S$.
\end{lem}

\begin{proof}[Proof of Lemma~\ref{lem:1stab}]
	By the independence of the samples $S_1, \ldots, S_m$ and the fact that conditioning reduces entropy, $$I\left(\cA(S); S\right)\geq \sum_{i=1}^m I\left(\cA\left(S\right); S_i\right)$$
	By the Cauchy-Schwartz inequality,
	$$\sum_{i=1}^m\sqrt{I\left(\cA\left(S\right); S_i\right)}\leq \sqrt{m \sum_{i=1}^m I\left(\cA\left(S\right); S_i\right))}.$$
\end{proof}

\begin{proof}[Proof of Lemma~\ref{lem:2stab}]
	\begin{align*}
	\sqrt{I\left(\cA\left(S\right); S_i\right)}&=\sqrt{\kl\bigl(\left(\cA\left(S\right), S_i\right) ~||~ \cA(S)\times S_i\bigr)}\\
	&\geq \tvd\bigl(\left(\cA\left(S\right), S_i\right) ~,~\cA(S)\times S_i\bigr)\tag{Pinsker's inequality}
	\\
	&=\tvd\bigl(\left(\cA\left(S^{(i, z)}\right), z\right) ~,~\cA(S)\times z\bigr)
	\\
	&=\ex{z}{\tvd\bigl(\cA\left(S^{(i, z)}\right),~\cA(S)\bigr)} .
	\end{align*}
	where the third step follows from the fact that $\left(\cA\left(S\right), S_i\right)$ and $\left(\cA\left(S^{(i, z)}\right), z\right)$ are identically distributed, and the fact that $S_i$ and $z$ are identically distributed.
\end{proof}

We are now ready to prove (\ref{stability-main}). Recall that for any example $S_i=\left(x_i, y_i\right)$, 
we have $\emperr\left(\cA\left(S\right); S_i\right)\triangleq \mathbf{1}\left(\cA\left(S\right)(x_i)\neq y_i\right)$, where $\cA(S)(x_i)$ denotes the label of the output hypothesis $\cA(S)$ on $x_i$. Let $z$ denote a fresh example $(x, y)\sim \DD$ independent of $S$. Let $\cU$ denote the uniform distribution over~$[m]$. 

The two lemmas above imply that 
$$\ex{i\sim\cU, ~z}{\tvd\left(\cA\left(S^{(i, z)}\right),~\cA(S)\right)} \leq \sqrt{\frac{d}{m}}.$$
It follows that for any $\tilde{z}\in\cX\times\{0,1\}$, we must have 
$$\ex{i\sim\cU, ~z}{\tvd\left(\emperr\left(\cA\left(S^{(i, z)}\right); \tilde{z}\right),~\emperr\left(\cA(S); \tilde{z}\right)\right)} \leq \sqrt{\frac{d}{m}},$$
which is equivalent to 
$$\ex{i\sim\cU, ~z}{\left\vert\ex{S, \cA}{\emperr\left(\cA\left(S^{(i, z)}\right); \tilde{z}\right)-\emperr\left(\cA(S); \tilde{z}\right)}\right\vert} \leq \sqrt{\frac{d}{m}},$$
which implies 
$$\ex{i\sim\cU, ~z}{\ex{S, \cA}{\emperr\left(\cA\left(S^{(i, z)}\right); S_i\right)-\emperr\left(\cA(S); S_i\right)}} \leq \sqrt{\frac{d}{m}}.$$
Finally, we use the fact that $\emperr\left(\cA\left(S^{(i, z)}\right); S_i\right)$ and $\emperr\left(\cA(S); z\right)$ are identically distributed to get 
$$\ex{i\sim \cU}{\ex{S, z, \cA}{\emperr\left(\cA\left(S\right); z\right)-\emperr\left(\cA\left(S\right); S_i\right)}}\leq \sqrt{\frac{d}{m}},$$
which leads directly to (\ref{stability-main}).

\subsection{Proof IV: PAC-Bayes}\label{app:gen-pac-bayes}\label{proof-iv}

The fourth proof is straightforward via a connection between information learners and the PAC-Bayes framework. 
The PAC-Bayes framework considers distributions over hypotheses. It is convenient to think of a distribution $\pp$ over hypotheses as a randomized hypothesis. We extend the notions of error to randomized hypotheses as follows
\[
\err(\qq; \DD) = \ex{h\sim \qq}{\err(h; \DD)}
\]
and 
\[
\emperr(\qq; S) = \ex{h\sim \qq}{\emperr(h; S)}.
\]

Fix some randomized hypothesis $\pp$. The following theorem known as the PAC-Bayes bound gives
a bound on the generalization error simultaneously for all randomized hypotheses $\qq$ in terms of their KL-divergence with $\pp$.

\medskip

\begin{theorem}[\citealt*{McAllster98, SS-text}]\label{thm:pac-bayes}
	Let $\DD$ be a distribution over examples, 
	and let $\pp$ be a fixed but otherwise arbitrary distribution over hypotheses. 
	Let $S$ denote a set of $m$ i.i.d.\ examples generated by $\DD$.
	Then, the following event occurs with probability at least~$1-\delta$:
	for every distribution $\qq$ over hypotheses,
	$$\err(\qq; \DD)-\emperr(\qq; S)\leq \sqrt{\frac{\kl\left(\qq || \pp\right)+\ln\left(m/\delta\right)}{m}}.$$
\end{theorem}

Following the Bayesian reasoning approach, the distribution $\pp$ can be thought of as the a priori output of the algorithm for the target concept, and after seeing the input sample $S$ the learning process outputs the distribution $\qq$ (which may depend on $S$), which is its a posteriori output for the target concept. The PAC-Bayes theorem bounds the generalization error of the algorithm in terms of the KL-divergence between the a priori and the a posteriori outputs.

Technically, the proof of this section  follows from expressing the mutual information between random variables $X,Y$
in terms of the KL-divergence between $X\vert Y$ and $X$: Let $\cA$ be a $d$-bit information learner for $\hh$, and let $\DD$ be a realizable distribution. Then
\[
d\geq I\bigl(S;\cA(S)\bigr) =\ex{S\sim\DD^m}{\kl\bigl(P_{h|S} || P_{h}\bigr)} 
\]
(recall that $P_{h}$ denotes the marginal distribution of~$\cA(S)$ for $S\sim\DD^m$). Therefore, by Markov's inequality, with probability $1-\delta$ it holds that 
\[
\kl\bigl(P_{h|S} || P_{h}\bigr)\leq d/\delta.
\]
Now we can apply the standard PAC-Bayes generalization bound (Theorem~\ref{thm:pac-bayes}) and deduce that with probability at least $1-\delta$ over the choice of the sample, the \emph{expected} generalization error is 
\[
O\left(\sqrt{\frac{d{/\delta}+\ln(m/\delta)}{m}}\right),
\] 
where the expectation is taken over $h\sim P_{h|S}$. Indeed, this follows by choosing the prior $\pp$ in the PAC-Bayes bound to be the marginal distribution $P_{h}$ of the output of $\cA$, and choosing the posterior $\qq$ to be the conditional distribution $P_{h|S}$. Hence, by rephrasing the statement of this bound we can obtain a form  similar to Theorem~\ref{thm:info-comp}. 

\medskip

\begin{theorem}
	\label{thm:gen-pac-bayes}
	Assuming $m\geq 5\frac{(d+1)}{\e^2}\ln\left(\frac{d+1}{\e}\right)$, for every $\e > 0$,
	\[
	\pr{S}{\ex{\cA}{\err(\cA(S)); \DD)-\emperr\left(\cA(S); S\right)} > \e}<\frac{d+1}{m\e^{2}}
	\]
	where the probability is taken over the randomness in the sample $S$ and the expectation inside the probability is taken over the random coins of $\cA$.
\end{theorem}

\section{Confidence Amplification}\label{apndx:boost}
\addcontentsline{toc}{section}{Appendix I: Confidence Amplification}

We show that the same standard procedure used for boosting the confidence of learning algorithms can be used with information learners at a modest cost.

\medskip

{\bf Theorem B.1} (Confidence amplification.) 
Let $\cA$ be a $d$-bit information learner, and let $\DD$ be a distribution on examples.
Assume that when $\cA$ recieves $m_0$ examples from $\DD$, then with probability $1/2$ 
it outputs a hypothesis with error at most $\eps$. 
Then there is an algorithm $\cB$ such that for every $\delta>0$,
\begin{itemize}
	\item When $\cB$ recieves 
	\[m = m_0 \lceil \log(2/\delta)\rceil +\frac{2\ln\left(4\log(2/\delta)/\delta\right)}{4\e^2}\]
	examples from $\DD$, then with probability at least $1-\delta$
	it outputs a hypothesis with error at most $\eps$.
	\item $\cB$ is a $\big(\log\log(2/\delta)+d\log(2/\delta)\big)$-bit information learner.
\end{itemize}

\begin{proof}
	We use the natural confidence amplification procedure.
	Set $k=\lceil \log(2/\delta) \rceil$, and draw $k$ subsamples $\left(S^{(1)}, \ldots, S^{(k)}\right)\triangleq S$, each of size $m_0$, and another ``validation'' set $T$ of size ${2\ln\bigl(4\log(2/\delta)/\delta\bigr)}/{\e^2}$. Run $\cA$ independently on each of the $k$ subsamples to output hypotheses $\vec h = (h_1, \ldots, h_k)$. Next, we validate $h_1, \ldots, h_k$ on $T$ and output a hypothesis $h^* \in\{h_1,\ldots,h_k\}$ with minimal empirical error on~$T$. 
	
	Item $1$ in the above theorem follows directly from standard analysis (Chernoff plus union bound). 
	To prove item $2$, first we note that $(S, h_1, \ldots, h_k)$ is independent of $T$. 
	So,
	\begin{align}
	I\left(h^*; S, T\right) &\leq I\left(h^*, \vec h; S,T \right)\nonumber\\
	&= I(\vec h ; S,T) + I(h^* ; S,T \vert \vec h)\nonumber\\
	&= I(\vec h ; S) + \bigl[I(h^* ; T \vert \vec h) + I(h^* ; S \vert T, \vec h)\bigr]\nonumber\\
	&\leq \sum_{i=1}^k I(h_i; S^{(i)})+ H\left(h^* \vert h_1, \ldots, h_k\right)  + 0 \nonumber\\
	&\leq d k + \log(k) \nonumber 
	\end{align}
\end{proof}


\begin{thebibliography}{1}
	
	
	
	\bibitem[Arutyunyan(1968)]{arutyunyan'68}
	E. A. Arutyunyan. 
	\newblock Bounds for the exponent of the probability of error for a semicontinuous memoryless channel.
	\newblock
	{\em Problems of Information Transmission} 4(4), pages 29-40, 1968.
	
	\bibitem[Bassily et. al.(2016)]{BNS+16}
	R. Bassily, K. Nissim, A. Smith, U. Stemmer, and J. Ullman.
	\newblock Algorithmic stability for adaptive data analysis.
	\newblock In {\em STOC}, pages 1046--1059, 2016.
	
	\bibitem[Bassily et. al. (2014)]{BST14}
	R. Bassily, A. Smith, and A. Thakurta.
	\newblock Private empirical risk minimization: Efficient algorithms and tight
	error bounds.
	\newblock In {\em {FOCS}}, pages 464--473, 2014.
	
	\bibitem[Beimel et. al.(2010)]{BeiNS10}
	A. Beimel,  S. P. Kasiviswanathan , and K. Nissim.
	\newblock  Bounds on the sample complexity for private
	learning and private data release.
	\newblock In {\em TCC}, pages 437--454, 2010.
	
	\bibitem[Beimel et. al.(2013)]{BeiNS13}
	A. Beimel, K. Nissim, and U. Stemmer.
	\newblock Characterizing the sample complexity of private learners.
	\newblock In {\em ITCS}, pages 1--10, 2013.
	
	\bibitem[Blumer et. al.(1987)]{occam} A. Blumer, A. Ehrenfeucht, D. Haussler, and M. Warmuth.
	\newblock Occam's razor. 
	\newblock {\em IPL} 24 377-380, 1987.
	
	
	\bibitem[Braverman and Garg(2014)]{log(I)} M. Braverman and A. Garg.
	\newblock Public vs private coin in bounded-round information. 
	\newblock  In {\em ICALP} (1), pages 502--513, 2014.
	
	
	\bibitem[Bun et. al.(2015)]{BNSV15}
	M. Bun, K. Nissim, U. Stemmer, and S.~P. Vadhan.
	\newblock Differentially private release and learning of threshold functions.
	\newblock In {\em FOCS}, pages 17--20, 2015.
	
	\bibitem[Cover and Thomas(2006)]{book} T. Cover and J. A. Thomas.
	{\em Elements of information theory.}
	Wiley-Interscience New York, 2006.
	
	\bibitem[De(2012)]{De11}
	A. De.
	\newblock Lower bounds in differential privacy.
	\newblock In {\em TCC}, pages 321--338, 2012.
	
	\bibitem[Dwork et. al.(2015)]{DFH+b15}
	C. Dwork, V. Feldman, M. Hardt, T. Pitassi, O. Reingold,
	and A. Roth.
	\newblock Generalization in adaptive data analysis and holdout reuse.
	\newblock In {\em NIPS}, 2015.
	
	\bibitem[Dwork et. al.(2015)]{DFH+15}
	C. Dwork, V. Feldman, M. Hardt, T. Pitassi, O. Reingold,
	and Aaron Roth.
	\newblock Preserving statistical validity in adaptive data analysis.
	\newblock In {\em STOC}, pages 117--126, 2015.
	
	\bibitem[Dwork et. al.(2006)]{DMNS06}
	C. Dwork, F. McSherry, K. Nissim, and A. Smith.
	\newblock 
	Calibrating noise to sensitivity in private data analysis.
	\newblock In {\em TCC}, pages 265--284, 2006.
	
	\bibitem[Dwork and Roth(2014)]{dp-text}
	C. Dwork and A. Roth.
	\newblock The algorithmic foundations of differential privacy.
	\newblock {\em Foundations and Trends in Theoretical Computer Science},
	9(3-4):211--407, 2014.
	
	\bibitem[Feldman and Xiao(2014)]{FX14}
	V. Feldman and D. Xiao.
	\newblock Sample complexity bounds on differentially private learning via
	communication complexity.
	\newblock In {\em COLT}, pages 1000--1019, 2014.
	
	\bibitem[Harsha et. al.(2010)]{SamRej}
	P. Harsha, R. Jain, D. McAllester, and J. Radhakrishnan.
	\newblock The Communication Complexity of Correlation.
	\newblock  {\em IEEE Trans. Information Theory} 56(1), pages 438-449, 2010.
	
	
	\bibitem[Haussler(1995)]{Hassler}
	David Haussler. 
	\newblock Sphere packing numbers for subsets of the boolean $n$-cube with bounded
	Vapnik-Chervonenkis dimension. 
	\newblock {\em J. Comb. Theory, Ser. A,} 69(2), pages 217--232, 1995.
	
	\bibitem[Kasiviswanathan et. al.(2008)]{KLNRS08}
	S. Kasiviswanathan, H.~K. Lee, K. Nissim, S. Raskhodnikova, and
	A. Smith.
	\newblock What can we learn privately?
	\newblock In {\em FOCS}, pages 531--540, 2008.
	
	\bibitem[Littlestone and Warmuth(1986)]{comp->learn}N. Littlestone and M. Warmuth. 
	\newblock Relating data compression and learnability. 
	\newblock {\em Unpublished}, 1986.
	
	\bibitem[McAllester(2003)]{McAllster98}
	D.~McAllester.
	\newblock {PAC}-bayesian model averaging.
	\newblock {\em Machine Learning Journal}, 5:5--21, 2003.
	
	\bibitem[McGregor et. al.(2010)]{MMP+10}
	A. McGregor, I. Mironov, T. Pitassi, O. Reingold, K. Talwar,
	and S. Vadhan.
	\newblock The limits of two-party differential privacy.
	\newblock In {\em FOCS}, pages 81--90, 2010.
	
	
	\bibitem[Moran and Yehudayoff(2016)]{learn->comp}S. Moran and A. Yehudayoff. 
	\newblock Sample compression schemes for VC classes. 
	\newblock {\em JACM} 63 (3), 2016.
	
	\bibitem[Raginsky and Xu(2017)]{RX17}
	M. Raginsky and A. Xu.
	\newblock Information-theoretic analysis of generalization capability of
	learning algorithms.
	\newblock {\em arXiv:1705.07809},
	2017.
	
	\bibitem[Rogers et. al.(2016)]{RRST16}
	R. Rogers, A. Roth, A. Smith, and O.~Thakkar.
	\newblock Max-information, differential privacy, and post-selection hypothesis
	testing.
	\newblock In {\em FOCS}, pages 487--494, 2016.
	
	\bibitem[Russo and Zhou(2016)]{Russo-Zhou16}
	D. Russo and J. Zhou.
	\newblock Controlling bias in adaptive data analysis using information theory.
	\newblock In {\em AISTATS}, pages 1232--1240, 2016.
	
	\bibitem[Shalev{-}Shwartz and Ben{-}David(2014)]{SS-text}
	S. Shalev{-}Shwartz and S. Ben{-}David.
	\newblock Understanding machine learning: From theory to algorithms.
	\newblock {\em Cambridge University Press}, 2014.
	
	\bibitem[Vadhan(2017)]{Vad_survey_DP16}
	S. Vadhan.
	\newblock The complexity of differential privacy.
	\newblock {\em Tutorials on the Foundations of Cryptography}, pages 347-450, 2017.
	
	
	
	
\end{thebibliography}
\end{document}